\def\year{2020}\relax
\documentclass[letterpaper]{article} 
\usepackage{aaai20}  
\usepackage{times}  
\usepackage{helvet} 
\usepackage{courier}  
\usepackage[hyphens]{url}  
\usepackage{graphicx} 
\usepackage{graphicx}  
\usepackage{xspace}
\usepackage{soul}
\usepackage{graphicx}

\usepackage{color, colortbl}
\definecolor{Gray}{gray}{0.9}

\title{Stochastic Fairness and Language-Theoretic Fairness in Planning on Nondeterministic Domains}

 \author{
 Benjamin Aminof\\
 TU Wien, Austria\\
 {aminof@forsyte.at}
 \And
 Giuseppe De Giacomo\\
 Univ.\ Roma ``La Sapienza'', Italy\\
 {degiacomo@diag.uniroma1.it}
 \And
 Sasha Rubin\\
 Univ. Sydney, Australia\\
 {sasha.rubin@sydney.edu.au}
 }

\usepackage{latexsym}

\usepackage{url}
\usepackage{enumerate}
\usepackage{tikz}
\usepackage{amsmath}
\usepackage{tabularx}
\usepackage{booktabs}
\usepackage{multirow}
\usepackage{algorithm}
\usepackage[noend]{algpseudocode}
\usepackage[inline]{enumitem}

\usepackage[normalem]{ulem}

\newcommand{\cut}[1]{}

\usepackage{changebar}

\usepackage{amssymb}

\newcommand{\A}{\mathcal{A}} 
\newcommand{\C}{\mathcal{C}} 
 \newcommand{\F}{\mathcal{F}}
 
 \newcommand{\J}{\mathcal{J}}

 \newcommand{\R}{\mathcal{R}}
 
\newcommand{\U}{\mathcal{U}}

\newcommand{\tup}[1]{\langle #1\rangle}            

\newcommand{\Math}[1]{\ensuremath{#1}}

\newcommand{\Nat}{\Math{\mathbb{N}}}

\newcommand{\true}{\mathtt{true}}
\newcommand{\false}{\mathtt{false}}

\newcommand{\mnext}{\bigcirc}		
\newcommand{\malways}{\Box}		
\newcommand{\meventually}{\Diamond}	
\newcommand{\muntil}{\mathop{\U}}	

\long\def\eatpar#1{%
\ifx#1\par                      
\let\nextmove=\eatpar           
\else
\let\nextmove=#1
\fi
\nextmove
}

\def\qed{\hfill{\qedboxempty}      
  \ifdim\lastskip<\medskipamount \removelastskip\penalty55\medskip\fi}

\def\qedboxempty{\vbox{\hrule\hbox{\vrule\kern3pt
                 \vbox{\kern3pt\kern3pt}\kern3pt\vrule}\hrule}}

\def\qedfull{\hfill{\qedboxfull}   
  \ifdim\lastskip<\medskipamount \removelastskip\penalty55\medskip\fi}

\def\qedboxfull{\vrule height 4pt width 4pt depth 0pt}

\newcommand{\Init}{\modeit{Init}}

\newcounter{bean}

\newenvironment{tightenumerate}{
                \begin{list}{
                  {\mbox {
                      \arabic{bean}.\/}}}{\usecounter{bean}
                      \setlength{\itemsep}{-1pt}\setlength{\topsep}{0pt}}}{
                \end{list}}

\newenvironment{tightitemize}{
                \begin{list}{$\bullet$}{
                    \setlength{\itemsep}{-1pt}}{\setlength{\topsep}{0pt}}}{
                \end{list}}

\newcommand{\under}[1]{\mbox{\underline{\it\smash{#1}\vphantom{\lower.05ex\hbox{
x}}}}}

\newcommand{\commentarea}[1]{}

\newcommand{\limp}{\supset}

\newcommand{\Next}{\raisebox{-0.27ex}{\LARGE$\circ$}}

\newcommand{\LTL}{{\sc ltl}\xspace}

\newcommand{\LTLf}{{\sc ltl}$_f$\xspace}

\renewcommand{\Nat}{{\rm I\kern-.23em N}}

\newcommand{\np}{{\sc np}\xspace}
\newcommand{\exptime}{{\sc exptime}\xspace}
\newcommand{\nexptime}{{\sc nexptime}\xspace}
\newcommand{\expspace}{{\sc expspace}\xspace}

\newsavebox{\fmbox}

\newtheorem{theorem}{Theorem}

\newtheorem{proposition}{Proposition}
\newtheorem{lemma}{Lemma}

\newtheorem{remark}{Remark}

\newtheorem{definitionAux}{Definition}

\newtheorem{claimAux}{Claim}

\newtheorem{exampleAux}{Example}
\newenvironment{example}{\begin{exampleAux}\rm}{\end{exampleAux}}

\newtheorem{examplesAux}{Examples}

\newtheorem{constructionAux}{Construction}

\newenvironment{proof}{\noindent \textsl{Proof.\ }}{\qedfull}

\long\def\eatpar#1{%
\ifx#1\par                      
\let\nextmove=\eatpar           
\else
\let\nextmove=#1
\fi
\nextmove
}

\def\qed{\hfill{\qedboxempty}      
  \ifdim\lastskip<\medskipamount \removelastskip\penalty55\medskip\fi}

\def\qedboxempty{\vbox{\hrule\hbox{\vrule\kern3pt
                 \vbox{\kern3pt\kern3pt}\kern3pt\vrule}\hrule}}

\def\qedfull{\hfill{\qedboxfull}   
  \ifdim\lastskip<\medskipamount \removelastskip\penalty55\medskip\fi}

\def\qedboxfull{\vrule height 4pt width 4pt depth 0pt}

\newcommand{{\incolumn}}[1]{\begin{tabular}[c]{c} #1 \end{tabular}}
\newcommand{{\incolumnmath}}[1]{\begin{array}[c]{c} #1 \end{array}}

\newcommand{\goal}{target\xspace}

\usetikzlibrary{arrows,automata,positioning}

\begin{document}

\maketitle

\begin{abstract}
We address two central notions of fairness in the literature of planning on nondeterministic fully observable domains. The first, which we call stochastic fairness, is classical, and assumes an environment which operates probabilistically using possibly unknown probabilities. The second, which is language-theoretic, assumes that if an action is taken from a given state infinitely often then all its possible outcomes should appear infinitely often (we call this state-action fairness). While the two notions coincide for standard reachability goals, they diverge for temporally extended goals. This important difference has been overlooked in the planning literature, and we argue has led to confusion in a number of published algorithms which use reductions that were stated for state-action fairness, for which they are incorrect, while being correct for  stochastic fairness. We remedy this and provide an optimal sound and complete algorithm for solving state-action fair planning for LTL/LTLf goals, as well as a correct proof of the lower bound of the goal-complexity (our proof is general enough that it provides new proofs also for the no-fairness and stochastic-fairness cases).  Overall, we show that stochastic fairness is better behaved than state-action fairness. 
\end{abstract}

\section{Introduction} \label{sec:introduction}

Nondeterminism in planning captures uncertainty that the agent has at planning time about the effects of its actions. For instance, ``remove block A from the table'' may either succeed, resulting in ``block A is not on the table'', or fail, resulting in ``block A is on the table''. Plans in nondeterministic environments are not simply sequences of actions as in classical planning; rather, the next action may depend on the sequences of actions (and observations\footnote{In this paper we assume there is no uncertainty about the current state of the system, i.e., environments are fully observable.}) so far, and are captured by policies (also known as strategies and controllers).

Broadly speaking, nondeterminism manifests in one of two ways, stochastic- and adversarial-environments.

\subsubsection{Stochastic environments}
Nondeterministic environments with probabilities are often modeled as Markov Decision Processes (MDPs) in planning. These are state-transition systems in which the probability of an effect depends only on the current state and action. 
However, sometimes the probabilities of action effects are not available, or are non stationary,  or are hard to estimate, e.g., a robot may encounter an unexpected obstacle, or an exogenous event or failure occurs. A long thread in this setting aims to understand what it means to plan in such an environment~\cite{DaTV99,PistoreT01,Cimatti03,Ghallab:CUP2016,DIppolitoRS18}. One common intuition is that the goal should be achievable by trial-and-error while expecting only a finite amount of bad luck~\cite{Cimatti03}, e.g., a policy that repeats the action ``remove block A from the table'' would eventually succeed under this assumption. This amounts to assuming that some unknown distribution assigns a non-zero probability to each of the alternative effects.{\footnote{Although reinforcement-learning also makes a similar assumption, it is out of the scope of this work which focuses on nondeterminism in model-based control and in planning in particular.}}
Thus, although there are no explicit probabilities, the stochastic principle is still in place, and we call such assumptions \emph{stochastic fairness}.
Plans in such a setting are called strong-cyclic, and their importance is evidenced by the fact that there are several tools for finding strong-cyclic policies, e.g., NDP \cite{AlfordKNG14}, FIP \cite{FuJNBY16}, myND \cite{DMattmullerOHB10}, Gamer \cite{KissmannE11}, PRP \cite{MuiseMB12}, GRENADE \cite{RamirezS14}, and FOND-SAT \cite{GeffnerG18}. Such policies also correspond to ensuring that the goal holds with probability one~\cite{Ghallab:CUP2016,GeBo13}.
\subsubsection{Adversarial environments}
Nondeterministic environments without probabilities are often modeled as fully observable nondeterministic planning domains (FOND). 
These are state-transition systems in which the effect of an action is a set of possible states, rather than a single state as in classical planning. Policies that guarantee success, i.e., the goal is achieved no matter how the nondeterminism is resolved, are called strong solutions. When handling adversarial nondeterminism it is often reasonable to require that a policy should guarantee success under some additional assumptions about the environment. For instance, a typical assumption is that repeating an action in a given state results in all possible effects, e.g., repeating the action ``remove block A from the table" would eventually succeed (as well as eventually fail). Note that this can be expressed as a property of traces, and so for the purpose of this paper, we call such notions \emph{language-theoretic fairness}. We focus on one central such notion which we call \emph{state-action fairness} and which says, of a trace, that if an action $a$ is taken from a state $s$ infinitely often in the trace, and if $s'$ is a possible effect of $a$ from $s$, then infinitely often in the trace $s'$ is the resulting effect of action $a$ from state $s$. Although there are many notions of fairness, this particular notion has been identified as providing sufficient assumptions that guarantee the success of solutions that repeatedly retry; see \cite{DIppolitoRS18} where the notion is called \emph{state strong fairness}.

\subsubsection{What is the relationship between fairness in an adversarial setting and fairness in a stochastic setting?}
On the one hand, the two notions of fairness are similar. Indeed, planning assuming either notion of fairness means that the policy can ignore some traces, which are guaranteed not to be produced by the environment.\footnote{In the language-theoretic setting, the policy need not succeed on traces that do not satisfy the fairness property; while in the stochastic setting the policy need not succedd on any set of traces whose probability measure is zero.} Also, it turns out that when planning for reachability goals (i.e., eventually reach a certain target set of states) the two notions of fairness are interchangeable. More precisely, a policy achieves the reachability goal assuming stochastic fairness (i.e., it is a strong-cyclic solution) if and only if it achieves the reachability goal assuming state-action fairness (i.e., the target set is reached on all state-action fair traces).
On the other hand, it turns out that the two notions of fairness are not generally interchangeable for planning for temporally extended goals (such as those expressed in linear temporal logic \LTL or its finite-trace variant \LTLf). It is the purpose of this paper to clarify this fact and study its consequences.

\subsubsection{Outline of the paper and contributions}
In Section~\ref{sec:fairness-comparison} we point out the distinction between stochastic fairness and state-action fairness in the context of planning. Once this distinction has been noted, one realizes that there are algorithms (published in IJCAI) for fair planning for temporally-extended goals that, although stated for state-action fairness, are actually correct for stochastic fairness (but do not address state-action fairness at all). The relevant parts of these algorithms are discussed in Section~\ref{sec:confusion}.
To remedy this, the focus of the rest of the paper is on algorithms and the computational complexity of planning for temporally-extended goals assuming state-action fairness.  

In Section~\ref{sec:algorithms-and-upper-bounds} we provide a new algorithm for this problem that does not conflate the two notions. We go on to show that the complexity in the goal is in $2$\exptime, while the complexity in the domain is in $1$\nexptime.  
 
In Section~\ref{sec:lower-bounds} we provide a proof of the matching $2$\exptime lower-bound for the goal-complexity. We also discuss the domain-complexity: it is $1$\exptime-hard already for reachability goals, leaving a gap between deterministic and nondeterministic exponential time.
We also show that our lower bound is proved using a technique that is general enough to give new proofs of the $2$\exptime-hardness for the goal complexity also for the no-fairness and stochastic-fairness cases.
   
\section{Fair Planning Problems} \label{sec:framework}
In this section we define planning domains, temporally extended goals, and isolate the two notions of fairness.

\subsubsection{Planning Domains}
A \emph{nondeterministic planning domain} is a tuple $(St,Act,s_0,Tr)$ where $St$ is a finite set of \emph{states}, $Act$ is a finite set of \emph{actions}, $s_0$ is an \emph{initial state}, and $Tr \subseteq St \times Act \times St$ is a \emph{transition relation}. We will sometimes write $Tr$ in functional form, i.e., $Tr(s,a) \subseteq St$. We say that the action $a$ is \emph{applicable} in state $s$ if $Tr(s,a) \neq \emptyset$. We assume, by adding a dummy action and state if needed, that for every state there is an applicable action.

For a finite set $X$ let $Dbn(X)$ denote the set of \emph{(probability) distributions} over $X$, i.e., functions $d:X \to [0,1]$ such that $\sum_{x \in X} d(x) = 1$. An element $x$ is in the \emph{support} of $d$ if $d(x) > 0$.
A \emph{stochastic planning domain} is a tuple $(D,Pr)$ where $D = (St,Act,s_0,Tr)$ is called the \emph{induced} nondeterministic planning domain, and $Pr$, called the \emph{probabilistic transition function}, is a partial function $Pr:St \times Act \to Dbn(St)$ defined only for pairs $(s,a)$ where $a$ is applicable in $s$, satisfying that the support of $Pr(s,a)$ is equal to $Tr(s,a)$. Note that stochastic domains are variants of Markov Decision Processes (MDPs). However, MDPs typically have Markovian rewards, while stochastic planning problems may have goals that depend on the history.

We will refer to both nondeterministic and stochastic planning domains simply as \emph{domains}. Unless otherwise stated, \emph{domains are compactly represented}, e.g., in variants of the Planning Domain Description Language (PDDL), and thus can usually be represented with a number of bits which is poly-logarithmic in the number of states and actions. In particular, the states are encoded as assignments to Boolean variables $\F$ called \emph{fluents}, thus we have that $St = 2^{\F}$. For symmetry, also the actions are encoded as assignments to Boolean variables $\A$ that are disjoint from $\F$, thus we have that $Act = 2^{\A}$. Although the literature also contains formalisms for compactly representing stochastic domains (such as Probabilistic PDDL),
here we will not be concerned with a detailed formalization of probabilistic transition functions since it is known (as we will later discuss) that probabilities essentially play no role in the stochastic-fair planning problem (formally defined below).

\subsubsection{Traces and Policies}
Let $D$ be a domain.
A \emph{trace} $\tau$ of $D$ is a finite or infinite sequence $(s_0 \cup a_0)(s_1 \cup a_1)\cdots$ over the alphabet $(St \cup Act) = 2^{\F \cup \A}$ where $s_0$ is the initial state, and $(s_{i-1},a_{i-1},s_i) \in Tr$ for all $i$ with $1 \leq i < |\tau|$. 
Moreover, the sequence $s_0 s_1 \cdots$ of states is called the \emph{path} induced by $\tau$. 
A \emph{policy} is a function $f:(St)^+ \rightarrow Act$ such that for every $u \in (St)^+$ the action $f(u)$ is applicable in the last state of $u$. 
Note that policies are history dependent in this paper.
A trace $\tau$ is \emph{generated by $f$}, or simply called an \emph{$f$-trace}, if for every finite prefix $(s_0 \cup a_0) \cdots (s_i \cup a_i)$ of $\tau$ we have that $f(s_0 s_1 \cdots s_i) = a_i$. 

A \emph{finite-state representation} of a policy $f$ is a finite-state input/output automaton that, on reading $u \in (St)^+$ as input, outputs the action $f(u)$. A \emph{finite-state policy} is one having a finite-state representation. 

A stochastic domain $D$ combined with a policy $f$ induces a (possibly infinite-state) Markov chain, denoted $(D,f)$, in the usual way, which gives rise to a probability distribution over the set of infinite $f$-traces in $D$~\cite{Vardi:FOCS85}. 

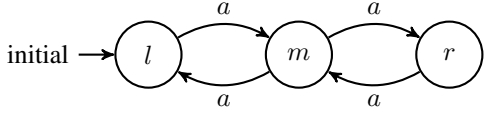
\begin{figure}
	\begin{center}
		\scalebox{1}{
\begin{tikzpicture}[node distance=2cm, auto,thick,>=stealth',initial text = {initial}]
\node[state,initial]  (sl) 			{$l$};
\node[state] 		(sm) [right of=sl] 	{$m$};
\node[state]		(sr) [right of=sm] 	{$r$};

\path[->] (sl) edge[bend left] node {$a$} (sm)
(sm) edge[bend left] node {$a$} (sl)
(sm) edge[bend left] node {$a$} (sr)
(sr) edge[bend left] node {$a$} (sm);            
\end{tikzpicture}
}
	\end{center}
	\caption{A drawing of the domain $D$ from Example~\ref{cex}, used in counterexamples. States are labeled by fluents, and there is a single applicable action.} \label{fig:fail-fairness}
\end{figure}

The following domain, illustrated in Figure~\ref{cex}, will be used in counterexamples. 
\begin{example} \label{cex}
	Define the domain $D = (St,Act,s_0,Tr)$ where $St = 2^\F$ with $\F = \{l,m,r\}$, $Act = 2^\A$ with $\A=\{a\}$, $s_0 = \{l\}$, and $Tr$ consist of the triples $(\{l\},\{a\},\{m\}),  (\{m\},\{a\},\{l\}), (\{m\},\{a\},\{r\})$ and $(\{r\},\{a\},\{m\})$. Note that the only applicable action (from any state) is $\{a\}$, only three states are reachable from the initial state using this action (i.e., $\{l\},\{m\}$ and $\{r\}$), and there is only one policy available (it always does the action $\{a\}$).
	Define the trace  $\tau$ as $(\{l,\mathrm{a}\} \{m,\mathrm{a}\} \{r,\mathrm{a}\} \{m,\mathrm{a}\})^\omega$.\footnote{For a finite string $u$, we write $u^\omega$ for the infinite string $u u u \dots$.} Note that this trace takes each of the transitions $m \xrightarrow{a} r$ and $m \xrightarrow{a} l$ infinitely often.
\end{example}

\subsubsection{Linear Temporal Logic}
Linear Temporal Logic (\LTL) is a formalism that was introduced into the verification literature for describing computations of programs without the use of explicit time-stamps~\cite{Pnueli77}. 
The logic has since been used in planning as a language for specifying temporally extended goals and for expressing search control, see, e.g., \cite{Fainekos:ICRA05,BacchusK00}. 

The syntax of \LTL consists of atoms $AP$, and is closed under the Boolean operations $\lnot$ and $\land$, and the temporal operators $\mnext$ (read ``next'') and $\muntil$ (read ``until''):
\[\psi ::= p \mid (\lnot \psi) \mid (\psi_1\land
\psi_2) \mid (\mnext\psi) \mid (\psi_1\muntil\psi_2)\]
with $p$ varying over the elements of $AP$.

We use the usual short-hands, e.g., $\false := p \land \lnot p$, $\psi_1 \limp \psi_2 := \lnot \psi_1 \lor \psi_2$, $\Diamond \psi  := \true \muntil \psi$ (read ``eventually $\psi$''), and $\Box \psi  := \lnot \Diamond \lnot \psi$ (read ``always $\psi$''). 

Formulas of \LTL are interpreted over infinite sequences $\tau = \tau_0 \tau_1 \cdots$ over the alphabet $2^{AP}$. Define $\tau,j \models \psi$ inductively on the structure of $\psi$, simultaneously for all time points $j \geq 0$, as follows:
\begin{itemize}
\item $\tau,j \models p$ if $p \in \tau_j$,
\item $\tau,j \models \psi_1 \land \psi_2$ if $\tau,j \models \psi_i$ for $i =1,2$,
\item $\tau,j \models \mnext \psi$ if $\tau,j+1 \models \psi$,
\item $\tau,j \models \psi_1 \muntil \psi_2$ if $\tau,k \models \psi_2$ for some $k \geq j$, and $\tau,i \models \psi_1$ for all $i \in [j,k)$.
\end{itemize}
 
We also consider the variant \LTLf of \LTL interpreted over finite sequences. It has the same syntax and semantics as \LTL except that $\tau$ is a finite sequence and that one defines $\mnext$ as follows, cf.~\cite{BacchusK00,BaierM06,DegVa13}: 
\begin{itemize} 
\item $\tau,j \models \mnext \psi$ if $j + 1 \leq last(\tau)$ and $\tau,j+1 \models \psi$ where $last(\tau)$ is the last position of $\tau$, i.e., $last(\tau) = |\tau| - 1$ since sequences start with position $0$.
\end{itemize}

If $\psi$ is an \LTL (resp. \LTLf) formula and $\tau$ is an infinite (resp. finite) sequence over $AP$, we write $\tau \models \psi$, and say that $\tau$ \emph{satisfies} $\psi$, iff $\tau,0 \models \psi$. 

We also make the following useful convention that allows us to interpret \LTLf formulas over infinite traces: if $\tau$ is infinite and $\psi$ is an \LTLf formula, then $\tau \models \psi$ is defined to mean that some finite prefix of $\tau$ satisfies $\psi$.

In the context of a planning domain $D$, we will take $AP$ to be $\F \cup \A$ (this is for convenience; some papers take $AP = \F$). We write $(D,f) \models A \psi$, and say that $f$ \emph{enforces} $\psi$, if every {infinite} $f$-trace of $D$ satisfies $\psi$.

\subsubsection{Planning Problems}
A \emph{goal} $G$ is a set of infinite traces of $D$. A \emph{planning problem} $\tup{D,G}$ consists of a domain $D$ and a goal $G$. \emph{Solving} the planning problem is to decide, given $D$ (compactly represented) and $G$ (suitably represented), if there is a policy $f$ such that every infinite $f$-trace satisfies $G$ (i.e., is in $G$). In this paper, goals will typically be represented by \LTL/\LTLf formulas.

\subsubsection{Fair Planning Problems}
We now define the two types of fair planning problems mentioned in the introduction.

A trace $\tau$ of a domain $D$ is \emph{state-action fair} if for every transition $(s,a,s')$ of $D$, if $s,a$ occurs infinitely often in $\tau$ then $s,a,s'$ occurs infinitely often in $\tau$. This can be expressed by the following \LTL formula :
\[
\phi_{D,fair} := \bigwedge_{(s,a,s') \in Tr} \left(\malways \meventually (s \wedge a) \limp \malways \meventually (s \wedge a \wedge \mnext s'))\right).
\]

A policy $f$ {solves} the \emph{state-action-fair planning problem} $\tup{D,\psi}$ if every state-action-fair $f$-trace satisfies $\psi$, written $(D,f) \models A^\text{sa-fair} \psi$. 

For a stochastic domain $D$, we write $(D,f) \models A^{=1} \psi$ to mean that the probability that an $f$-trace satisfies $\psi$ is equal to $1$, and we say that $f$ \emph{almost surely enforces} $\psi$.
It is known that
$(D,f) \models A^{=1} \psi$ does not depend on the probabilistic transition function of $D$, but only on its induced nondeterministic domain; indeed, it does not depend on the exact distributions $Pr(s,a)$ but only on their supports, which are specified by the transition relation $Tr$ of the induced nondeterministic domain, cf.~\cite{Vardi:LICS86}. Hence, we can actually extend this probabilistic notion of enforcing also to nondeterministic domains, as follows. For a nondeterministic domain $D$, we write $(D,f) \models A^{=1} \psi$ to mean that $(D',f) \models A^{=1} \psi$ where $D'$ is any stochastic domain whose induced nondeterministic domain is $D$. Thus, for a domain $D$ (nondeterministic or stochastic), we say that a policy $f$ solves the \emph{stochastic-fair planning problem} $\tup{D,\psi}$ if $(D,f) \models A^{=1} \psi$.

\subsubsection{Connection with Planning for Reachability Goals}
The classic goal in planning is reachability, typically represented as a Boolean combination $\goal$ of fluents, i.e., it can be expressed by an \LTL/\LTLf formula $\meventually \goal$. A policy $f$ enforcing $\meventually \goal$ is known as a \emph{strong solution}~\cite{Cimatti03} or an \emph{acyclic safe solution}~\cite{Ghallab:CUP2016}). A policy enforcing $\meventually \goal$ assuming state-action fairness is known as a \emph{strong cyclic solution}~\cite{Cimatti03} or a \emph{cyclic safe solution}~\cite{Ghallab:CUP2016}).

\subsubsection{Computational Complexity}
Planning problems have two inputs: the domain (represented compactly) and the goal (typically represented as a formula).
\emph{Combined complexity} measures the complexity in terms of the size of both inputs, while \emph{goal complexity} (resp. \emph{domain complexity}) only measures the complexity in the size of the goal (resp. {domain}). Formally, we say that the \emph{goal complexity is in a complexity class $\C$} if for every domain $D$, the complexity of the problem that takes as input a goal $\psi$ and decides if there is a solution to the planning problem $\tup{D,\psi}$, is in $\C$; and we say that the \emph{goal complexity is hard for $\C$} if there is a domain $D$ such that the complexity of the problem that takes as input a goal $\psi$ and decides if there is a solution to the planning problem $\tup{D,\psi}$ is $\C$-hard. Similar definitions hold for domain complexity. Such measures were first introduced in database theory~\cite{Vardi:STOC1982}.

\subsubsection{Automata-theoretic approach to planning} A typical approach for solving planning problems with temporally-extended goals is to use an automata-theoretic approach. Here we recall just enough for our needs in Sections~\ref{sec:confusion} and~\ref{sec:algorithms-and-upper-bounds}. 

A \emph{deterministic automaton} is a tuple $M = (\Sigma,Q,q_0,\delta,\C)$ where $\Sigma$ is the \emph{input alphabet}, $Q$ is a finite set of \emph{states}, $q_0 \in Q$ is the \emph{initial state}, $\delta:Q \times \Sigma \to Q$ is the \emph{transition function}, and $\C$ is the \emph{acceptance condition} (described later). A (finite or infinite) input word $u = u_0 u_1 \cdots$ determines a run, i.e., the sequence $q_0 q_1 \cdots$ of states starting with the initial state and respecting the transition function, i.e., $\delta(q_{i-1},u_{i-1}) = q_i$ for all $1 \leq i < |u|$. A word is \emph{accepted} by $M$ if its run satisfies the acceptance condition $\C$. There are a variety of different ways to define the acceptance condition. If $M$ is to accept only finite words, then we typically have $\C \subseteq Q$; and we say that a finite run satisfies $\C$ if its last state is in $\C$. Such an automaton is called a \emph{deterministic finite word automaton (DFW)}. If $M$ is to accept only infinite words, then there are a number choices for $\C$. We will not be concerned with the specific choice until Section~\ref{sec:algorithms-and-upper-bounds}.

The \emph{synchronous product} of a domain $D$ and a deterministic automaton $A$ over the input alphabet $2^{\F \cup \A}$ is a
domain, denoted $D \times A$, whose states are pairs $(d,q)$ where $d$ is a state of $D$ and $q$ is a state of $A$, and that can transition from state $(d,q) $ to state $(d',q')$ on action $a$ if $(d,a,d') \in Tr$ and the automaton can go from $q$ reading $d \cup a$ to $q'$. Intuitively, $D \times A$ simulates both $D$ and $A$ simultaneously. Such products are used in algorithms for planning with \LTL/\LTLf goals in Section~\ref{sec:confusion} and Section~\ref{sec:algorithms-and-upper-bounds}. We remark that the product is sometimes also compactly represented, although the details depend on the context and will not concern us.

\section{Stochastic Fairness $\not \equiv$ State-action Fairness} \label{sec:fairness-comparison}

In this section we compare the two notions of fairness in the context of planning. It turns out that they are equivalent for reachability goals, but not for general \LTL/\LTLf goals. The first principle is known, e.g.~\cite{Rintanen:ICAPS04,Ghallab:CUP2016}, and is repeated here for completeness.

\begin{proposition} \label{prop:equivalent for reachability goals} 
Let $D$ be a (nondeterministic or stochastic) domain and let $target$ be a Boolean combination of fluents. The following are equivalent for every finite-state policy $f$:
\begin{enumerate}
 \item $(D,f) \models A^\text{sa-fair}(\meventually \goal)$, i.e.,  the \goal is reached on state-action fair traces.
 \item $(D,f) \models A^{=1}(\meventually \goal)$, i.e., the \goal is reached with probability one.
\end{enumerate}
\end{proposition}

\begin{proof}
Assume that $(D,f) \models A^\text{sa-fair}(\meventually \goal)$. Observe that the state-action fair traces have probability $1$,
cf.~\cite{Vardi:LICS86}, and thus, by definition, $(D,f) \models A^{=1}(\meventually \goal)$.
For the other direction, assume by way of contradiction that $2.$ holds but $1.$ doesn't, and pick an infinite state-action fair $f$-trace $\tau$ that doesn't satisfy $\meventually \goal$. Let $M$ be the finite-state Markov chain induced by $D$ and $f$, viewed as a directed graph, and let $\pi$ be the path in $M$ induced by $\tau$. Since $\tau$ is state-action fair, $\pi$ reaches a bottom strongly connected component $C$ of $M$, and visits every state in $C$. By the assumption that $\tau \not \models \meventually \goal$, $\pi$ contains no state in which $\goal$ holds.
Let $\rho$ be some (fixed) prefix of $\pi$ that ends in a state in $C$, and consider the set $E$ of infinite
$f$-traces whose induced paths have $\rho$ as a prefix. Observe that the probability of $E$ is positive, and none of the traces in $E$ satisfy $\meventually \goal$. This contradicts $2$.
\end{proof}

We now turn to goals expressed as \LTL/\LTLf formulas. Unfortunately, in this case the analogue of Proposition~\ref{prop:equivalent for reachability goals} does not hold. Indeed, only the forward direction holds.

\begin{proposition}  \label{prop:underapprox}
Let $D$ be a domain, $\psi$ and \LTL/\LTLf formula, and $f$ a finite-state policy. If $(D,f) \models A^{\text{sa-fair}}(\psi)$ then $(D,f) \models A^{=1}(\psi)$.
\end{proposition}
\begin{proof}
As in Proposition~\ref{prop:equivalent for reachability goals}, simply use the fact that the set of infinite state-action fair $f$-traces has probability $1$.
\end{proof}

The next proposition shows that the converse of Proposition~\ref{prop:underapprox} does not hold.
Intuitively, the reason is that, assuming stochastic fairness, every finite trace that is enabled infinitely often appears with probability 1, while assuming state-action fairness, this is only true for traces of length one.

For the next proposition, recall Example~\ref{cex}.
\begin{proposition} \label{prop:cex}
There is a domain $D$, a finite-state policy $f$, and an \LTLf goal $\psi$ such that
$(D,f) \models A^{=1}(\psi)$, but for no policy $g$ does it hold that $(D,g) \models A^{\text{sa-fair}}(\psi)$.
\end{proposition}

\begin{proof}
Let $D$ be the domain from Example~\ref{cex}. Let $\psi$ be the \LTL/\LTLf formula $\meventually (l \land \mnext \mnext l)$ (i.e., eventually $l$ and two steps afterwards $l$ again). There is only one policy $f$ available: it always chooses the action $\{a\}$. Observe that $(D,f) \models A^{=1}(\psi)$, but that $(D,f) \not \models A^{\text{sa-fair}}(\psi)$ as witnessed by the trace $\tau := (\{l,\mathrm{a}\} \{m,\mathrm{a}\} \{r,\mathrm{a}\} \{m,\mathrm{a}\})^\omega$.
\end{proof}

\section{Confusion in the literature} \label{sec:confusion}
Certain algorithms in the literature for solving state-action fair planning problems with temporally extended goals rely on a reduction to another state-action fair planning problem, that, as we prove, is complete but not sound. The papers, in order of publication, are \cite{PatriziLG13}[Theorem $3$], \cite{GiacomoRubin:IJCAI18}[Theorem $4$] and \cite{Camacho:IJCAI19}[Theorem $2$]. If one assume stochastic-fairness instead of state-action fairness, then the reduction is both sound and complete. This suggests that the cited algorithms are correct if one assumes stochastic fairness instead of state-action fairness.

\subsubsection{The reduction}
We begin by describing the reduction without any mention of fairness.
From a planning problem $\tup{D,\psi}$, first define a deterministic automaton $A_\psi$ that recognizes exactly the traces that satisfy $\psi$. 
Second, define the domain $D' = D \times A_\psi$ as the synchronous product of $D$ and $A_\psi$. Finally, define the planning problem $\tup{D',Acc}$ where $Acc$ is a goal that captures the acceptance condition of $A_\psi$, i.e., $Acc$ consists of those traces of $D'$ whose first components are traces of $D$ that are accepted by $A_\psi$.\footnote{The representation of $Acc$ is induced by the acceptance condition of $A_\psi$, but here the specific representation of $Acc$ is not relevant.}

\subsubsection{Analsis of the reduction} If this reduction is to be used to give an exact algorithm for planning assuming state-action fairness, it should be sound and complete, i.e., $\tup{D,\psi}$ is solvable assuming state-action fairness iff $\tup{D',Acc}$ is solvable assuming state-action fairness. The reduction is indeed complete because every state-action fair trace in the product domain $D'$ projects to a state-action fair trace in $D$ (this follows immediately from the definition of state-action fairness and of the synchronous product). On the other hand, the reduction is not sound because there may be fair traces in $D$ that do not induce any fair trace in $D'$ (intuitively, this is due to synchronization in $D'$ between the domain $D$ and the automaton $A_\psi$). We formalise this in the following theorem which actually shows that the reduction is not sound no matter which deterministic automaton $A_\psi$ for $\psi$ is used.

\begin{theorem} \label{thm:state-action fairness and product}
There is a domain $D$, and an \LTL/\LTLf goal $\psi$, s.t. a) there is no solution to the state-action fair planning problem $\tup{D,\psi}$, but b) for every deterministic automaton $A_\psi$ accepting exactly the traces that satisfy $\psi$, there is a solution to the state-action fair planning problem $\tup{D',Acc}$, where $D'$ is the product of $D$ and $A_\psi$, and $Acc$ captures the acceptance condition of $A_\psi$.
\end{theorem}

\begin{proof}
Let $D$ (resp. $\tau$) be the domain (resp. trace) from Example~\ref{cex}, let $\psi$ be the formula
$\lnot l \lor \meventually (l \land \mnext \mnext \lnot r) \lor \meventually (l \land \mnext \mnext \mnext \mnext \lnot l)$, and observe that all traces of $D$ satisfy $\psi$ except for the trace $\tau$. 

There is a single policy $f$ available in $D$, i.e., always perform the single applicable action. However, $\tau$ is a state-action-fair $f$-trace that does not satisfy $\psi$. Thus, there is no solution to the state-action fair problem $\tup{D,\psi}$.

We claim that the single policy available in $D'$ is a solution to $\tup{D',Acc}$. For this, it is enough to show that every state-action fair trace in $D'$ induces in $D$ a trace that satisfies $\psi$, i.e., a trace other than $\tau$. Let $\tau'$ be a trace in $D'$ that induces $\tau$. To see that $\tau'$ is not state-action fair, let $(m,s)$ be a state that appears in $\tau'$ infinitely often after a state of the form $(l,?)$. Note that $(m,s)$ never appears as a source of a transition to a state of the form $(l,?)$. Indeed, since $l$ occurs on $\tau$ exactly every four steps, the source of such a transition is only reached three steps after reading an $l$; and while reading $\tau$, $A_\psi$ is always in a different state than $s$ three steps after reading an $l$ (so not to confuse occurrences of $\lnot l$ four steps after an $l$ with ones two steps after it). Thus, some successor $(l,q)$ of $(m,s)$ is enabled infinitely often but never taken.
\end{proof}

We note, however, that if one uses stochastic fairness instead of state-action fairness then the reduction above is sound and complete. This is because stochastic-fairness is preserved by taking a product with a deterministic automaton, a fact which is exploited in the automata-theoretic approach to verification of probabilistic systems~\cite{Vardi:FOCS85,Courcoubetis:JACM95,Bianco:FSTTCS95,Bollig:VSS04}:
\begin{theorem} \label{thm:product stochastic}
Let $\tup{D,\psi}$ be a planning problem, and let $\tup{D',Acc}$ be a planning problem constructed as in the reduction above. There is a policy solving $\tup{D,\psi}$ assuming stochastic fairness iff there is a policy solving $\tup{D',Acc}$ assuming stochastic fairness.
\end{theorem}

In summary, we conjecture that some errors in the proofs and algorithms for state-action fair planning in the literature arise from the mistaken intuition that state-action fairness always behaves like stochastic fairness, which it does not in the presence of even simple \LTL/\LTLf formulas (that are not reachability formulas).

\section{Algorithm for State-action Fair Planning} \label{sec:algorithms-and-upper-bounds}

In the previous section we showed that some algorithms in the literature for state-action fair planning for temporally extended goals use complete but unsound reductions. In this section, we provide a sound and complete reduction to the problem of solving \emph{Rabin games} (defined below). 

\begin{theorem} \label{thm:upper bounds}
	The combined (and thus goal) complexity of solving planning with \LTL/\LTLf goals assuming state-action fairness is in $2$\exptime, and the domain complexity is in $1$\nexptime\ (in the size of a compactly represented domain).
\end{theorem}

The main approach to solving such a problem is to use, explicitly or implicitly, an automata-theoretic approach. However, as we now remark, 
naive applications of this approach yield a $3$\exptime domain-complexity (which we then show how to lower to $1$\nexptime), a $3$\exptime combined-complexity (which we then show how to lower to $2$\exptime), and a $2$\exptime goal complexity.

\begin{remark}
The problem of solving the state-action fair planning problem $\tup{D,\psi}$ where $\psi$ is an \LTL/\LTLf formula is equivalent to solving the planning problem $\tup{D,\phi_{D,fair} \limp \psi}$ where \[
\phi_{D,fair} := \bigwedge_{(s,a,s') \in Tr} \left(\malways \meventually (s \wedge a) \limp \malways \meventually (s \wedge a \wedge \mnext s'))\right)
\]
is an \LTL formula expressing state-action fairness in the domain $D$ (for more on this equivalence see~\cite{Aminof:ICAPS19}). However, the size of $\phi_{D,fair}$ is \emph{exponential} in the size of $D$ (compactly represented). Thus, we have reduced the problem to solving planning for an \LTL goal of size exponential in the size of $D$ and linear in the size of $\psi$. In turn, there are algorithms that solve planning with \LTL goals (no fairness assumptions) that run in $1$\exptime in the size of the domain and $2$\exptime in the size of the goal~\cite{Aminof:ICAPS19,Camacho:ICAPS19}. Putting this together results in an algorithm for the state-action-fair planning problem that runs in $3$\exptime in the size of the domain $D$ and $2$\exptime in the size of the formula $\psi$. \end{remark}

The main insight that achieves the complexities in Theorem~\ref{thm:upper bounds} is that one should use Rabin conditions.  

A \emph{Rabin condition} over a set $X$ is a set $\R$ of pairs of the form $(I,F)$ with $I,F \subseteq X$. The pairs are called \emph{Rabin pairs}. An infinite sequence $\tau$ over the alphabet $X$ is said to \emph{satisfy the Rabin condition} $\R$ if there is a pair $(I,F) \in \R$ such that some $x \in I$ appears infinitely often in $\tau$ and no $x \in F$ appears infinitely often in $\tau$.\footnote{The reader might find it helpful to read the Rabin condition in \LTL notation: $\bigvee_{(I,F) \in \R} \malways \meventually I \land \lnot  \malways \meventually F$.}
Below we use Rabin conditions in two ways: as acceptance conditions (for automata) and as winning conditions (in games).

\subsubsection{Rabin Automata} 
A \emph{Deterministic Rabin Word (DRW) automaton} is an automaton $M = (\Sigma, Q,q_0,\delta,\R)$ where the acceptance condition $\R$ is a Rabin condition over $Q$. The \emph{size} of a DRW is the number of its states and its \emph{index} is the number of pairs in $\R$.

The reader may be wondering why we chose the Rabin acceptance condition instead of some other acceptance condition. The reason is: they can capture very general properties, including \LTL/\LTLf; they are naturally closed under union; they can naturally express that a trace is not state-action fair.

\begin{theorem} \label{thm:LTL to DRW} [cf. \cite{Vardi:Banff95}] 
	Given an \LTL/\LTLf formula $\psi$ one can build a DRW $M_\psi$ that accepts exactly the infinite traces satisfying $\psi$.\footnote{Recall that we define that an infinite trace satisfies an \LTLf formula $\psi$ if some prefix of it satisfies $\psi$.} Moreover, $M_\psi$ has size $2$exp and index $1$exp in $|\psi|$.
\end{theorem}

\begin{lemma} \label{lem:unfair DRW}
	Given a domain $D$ one can build a DRW $M_{D,unfair}$ that accepts exactly the infinite traces of $D$ that are not state-action fair. Moreover, $M_{D,unfair}$ has size and index $1$exp in the size of $D$ (compactly represented).
\end{lemma}
To see this, let the states of the DRW store the last state-action and last state-action-state of $D$, and the Rabin pairs are of the form $(\{sa\}, \{sas'\})$ for $Tr(s,a,s')$.

\begin{lemma} \label{lem:DRW union}
	Given DRW $M_1,M_2$ one can build a DRW $M$, denoted $M_1 \lor M_2$, that accepts the words accepted by $M_1$ or $M_2$. The size of $M$ is the product of the sizes of the $M_i$s, and the index of $M$ is the sum of the indices of the $M_i$s.
\end{lemma}
To see this, if $M_i = (\Sigma,Q_i,q_i, \delta_i,\R_i)$ define $M = (\Sigma,Q_1 \times Q_2, (q_1,q_2),\delta',\R')$ where $\delta'((s_1,s_2),\sigma)= (\delta_1(s_1,\sigma), \delta_2(s_2,\sigma))$, and $\R$ consists of all pairs of the form $(Q_1 \times I,Q_1 \times F)$ for $(I,F) \in \R_2$ and all pairs of the form $(I \times Q_2,F \times Q_2)$ for $(I,F) \in \R_1$.

\subsubsection{Rabin Games}
The other use for the Rabin condition is to give winning conditions in games. A \emph{Rabin game} is an \underline{explicitly} represented planning problem whose goal is expressed as a Rabin condition $\R$ over the set $St$ of states.

\begin{theorem} \cite{Buhrke:TACAS96,Emerson:FOCS88} \label{thm:solving rabin games}
There is an algorithm that solves Rabin games in time $O(d!n^dm)$ where $d$ is the number of Rabin pairs, $n$ is the number of states, and $m$ is the number of transitions. In addition, solving Rabin games is \np-complete (in the size of the explicit representation).
\end{theorem}

\subsubsection{Reduction and Algorithm}
We can now describe the algorithm promised in Theorem~\ref{thm:upper bounds}. Given a state-action fair planning problem $\tup{D,\psi}$, reduce it to the problem of solving the Rabin game $G = (Ar,Acc)$ constructed as follows. The arena $Ar$ is defined as the synchronous product of the domain $D$, explicitly represented, and the DRW $M = M_{D,unfair} \lor M_\psi$. The Rabin winning condition $Acc$ is induced by the Rabin acceptance condition $\R$ of $M$, i.e., $Acc$ consists of all pairs of the form $(S \times I, S \times F)$ for $(I,F) \in \R$.

This completes the description of the reduction. To see that it is sound and complete, simply note that a policy $f$ solves the state-action fair planning problem $\tup{D,\psi}$ iff every fair $f$-trace in $D$ is accepted by the DRW $M$ iff every trace (fair and not-fair) in $G$ generated by the strategy that maps $(s_0,q_0)(s_1,q_1) \cdots (s_n,q_n)$ to the action $f(s_0 s_1 \cdots s_n)$ satisfies the Rabin condition $Acc$. The first iff is due to Theorem~\ref{thm:LTL to DRW}, and the second iff follows from the definition of Rabin condition and of the synchronous product. 

For the complexity analysis, simply note that the DRW $M_{D,unfair} \lor M_\psi$ has size $1$exp in the size of $D$ (compactly represented) and $2$exp in $|\psi|$, and index $1$exp in $D$ (compactly represented) and $1$exp in $|\psi|$. Now apply Theorem~\ref{thm:solving rabin games} to get the stated goal, combined, and domain complexities.

\section{Lower bounds for state-action fair planning} \label{sec:lower-bounds}

We showed that state-action fair planning for temporally extended goals has $2$\exptime combined-complexity and goal-complexity, and $1$\nexptime domain complexity. In this section we study lower bounds for the problem and show that we can match the $2$\exptime goal complexity (with a technique that also supplies new proofs of $2$\exptime goal complexity for the cases of no-fairness and stochastic fairness). For domain-complexity, we observe that existing results show the problem is $1$\exptime-hard. This leaves open whether the domain complexity can be lowered from $1$\nexptime to $1$\exptime.

\subsubsection{Domain-complexity} 
It is not hard to establish a $1$\exptime lower-bound for the domain complexity. Indeed, one can reduce the problem of stochastic-fair planning with reachability goals, which is known to be $1$\exptime-complete~\cite{Littman:AAAI97,Rintanen:ICAPS04}. Indeed, introduce a fresh fluent $p$ and fix the goal $\meventually p$. Then, for a stochastic-fair planning problem with domain $D$ and reachability goal $\meventually target$, build a new domain $D_p$ from $D$ by adding the fluent $p$ and a new action with precondition $target$ and postcondition $p$. Then the stochastic-fair problem $\tup{D,\meventually target}$ has a solution iff the stochastic-fair problem $\tup{D_p,\meventually p}$ has a solution. Moreover, the latter holds iff it has a finite state solution. By Proposition~\ref{prop:equivalent for reachability goals}, this is equivalent to the fact that the state-action fair problem $\tup{D_p,\meventually p}$ has a solution.

\subsubsection{Goal complexity}

The contribution of this section is a proof of the following theorem.\footnote{The result is stated in~\cite{GiacomoRubin:IJCAI18} for \LTLf but with an incorrect proof. The error there is concluding that every $f$-trace visits each state at most once. This is true for memoryless strategies, but need not be true for other strategies which might be required when planning for temporally extended goals.} 

\begin{theorem} \label{thm:lower bounds}
The goal complexity (and therefore, also combined complexity) of planning for \LTL/\LTLf goals assuming state-action fairness is $2$\exptime-hard.
\end{theorem}

Inspired by \cite{Courcoubetis:JACM95},
we provide a polynomial-time construction that, given an alternating \expspace\ Turing machine $M$ and an input word $x$, produces a probabilistic domain $D$ (explicitly represented) and an \LTL formula $\Phi$ such that $M$ accepts $x$ iff $\exists f. (D,f) \models A \Phi$. Note that to handle the goal complexity, the domain $D$ will be independent of $M$ and $x$.

\def\bin{\textrm{bin}}
\def\Pos{\textsc{Pos}}
\def\Sym{\textsc{Sym}}
\def\Trn{\textsc{Trn}}

\def\Cell{\texttt{Cell}}

\def\conf{\texttt{conf}}

\def\Init{\texttt{Init}}
\def\Acc{\texttt{Acc}}
\def\Next{\texttt{Next}}
\def\block{\texttt{block}}

\emph{Notation.} An alternating Turing machine is a tuple $(Q,\Sigma,\Delta,q_0,q_a,q_r)$ where $Q$ is the set of states partitioned into $Q_\exists$ and $Q_\forall$ (called the existential and universal modes), $\Sigma$ is the tape-alphabet, $\Delta \subseteq (\Sigma \times Q)^2 \times \{L,R,N\}$ is the transition relation, and $q_0 \in Q$ is the initial state, $q_a,q_r \in Q$ are the accepting and rejecting states. A \emph{configuration} is a string matching the expression $\Sigma^* \cdot (\Sigma \times Q) \cdot \Sigma^*$; it is \emph{initial} (resp. \emph{accepting}, \emph{rejecting}) if the state is $q_0$ (resp. $q_a,q_r$).
A computation of $M$ is a sequence of configurations, starting in an initial configuration, respecting the transition relation, and ending in an accepting or rejecting state. Wlog, we assume that the existential and universal modes of $M$ strictly alternate, with the existential going first.

Say $M$ runs in space $2^{p(|x|)}$ for some polynomial $p(\cdot)$. In particular, a configuration of $M$ running on $x$ has length at most $2^{p(|x|)}$. Let $n := p(|x|)$. Intuitively, the domain $D$ ensures that the agent and the environment generate strings of the form
\begin{align*}
 C_0 \cdot & (\# \cdot T_1 \cdot \#' \cdot C_1 \cdot \#'' \cdot K_1) \cdot (\# \cdot T_2 \cdot \#' \cdot C_2 \cdot \#'' \cdot K_2)  \\
 & \cdots (\# \cdot T_j \cdot \#' \cdot C_j \cdot \#'' \cdot K_j) \cdot \# \cdot \bot \cdot \bot \cdot \bot \cdots
\end{align*}
where the $C_i$s are arbitrary strings over $\{0,1,\%,\$\}$, the $T_i$s and $K_i$s are arbitrary strings over $\{0,1\}$, and $\bot$ is a special symbol. Intuitively, the $C_i$s will encode configurations of $M$ and are generated by the agent, the $T_i$s will encode transitions of $M$ and are generated by the agent for odd $i$ and the environment for even $i$, and the $K_i$s are generated by the environment and encode a position/index $k \in [1,2^n]$ on the tape that the environment wants to check. Finally, $\bot$ holds in a sink of the domain that the agent can go to when it is done. Note that this allows the agent to never go to the sink, but such traces will be rejected by the goal formula. We define the \LTLf goal $\Phi := \Phi_{Env} \limp \Phi_{Ag}$. Intuitively, $\Phi$ will enforce that as long as the environment encodes its parts correctly (i.e., $\Phi_{Env}$ holds), then so does the agent, and the accepting state is reached (i.e., $\Phi_{Ag}$ holds). The formula $\Phi_{Ag} := \Phi_{conf} \land \Phi^{odd}_{tran} \land \Phi_{chal} \land \Phi_{acc}$, and $\Phi_{Env} := \Phi_{num} \land \Phi^{even}_{tran}$, where $\Phi_{conf}$ says that each $C_i$ encodes a configuration, with $C_0$ encoding the initial configuration; $\Phi^{odd}_{tran}$ (resp. $\Phi^{even}_{tran}$) says that each $T_i$ with $i$ odd (resp. even) encodes a transition of $M$; $\Phi_{num}$ says that each $K_i$ encodes a number in $[1,2^n]$; and $\Phi_{acc}$ says that an accepting configuration is reached; $\Phi_{chal}$ (think of it as a ``challenge'') is used to check that the $k$th letter in the configuration encoded by $C_i$ is the result of applying transition $T_i$ to the configuration $C_{i-1}$. Intuitively, since the environment is adversarial, all possible positions will be challenged and all universal transitions will be taken. Thus, the agent will be able to enforce the goal iff $M$ accepts $x$. We now provide details on how to write the subformulas of $\Phi$.

Choose a sufficiently large integer $m$ to encode all members of $Q$ and $\Sigma \times Q$ as a binary string of length exactly $m$.
Let $\Sym$ denote a set of binary strings of length $m$ that encode either a tape-letter $l$ or a tape-letter/state pair $(l,q)$. Let $\bin(i)$ denote the binary string of length $n$ whose numeric value is $i$. The possible configurations of $M$ are encoded by the strings of the form
$s(\% \cdot \bin(0) \cdot \$ \cdot \Sym) \cdot (\% \cdot \bin(1) \cdot \$ \cdot \Sym) \cdots (\% \cdot \bin(2^n-1) \cdot \$ \cdot \Sym)$ which have exactly one symbol encoding a tape-letter/state pair $(l,q)$. The reason for the $\bin(i)$s is they allow the formula to check if an encoding of one configuration can be reached in one step of $M$ from an encoding of another. We call the substring $(\% \cdot \bin(i) \cdot \$ \cdot w)$ the $i$th block, where $i$ is the \emph{block number} and $w$ is the \emph{block symbol}. One can write an \LTLf formula $\conf$, of size linear in $n$ and the size of $M$, that enforces this structure. Indeed, using a standard encoding of the binary counter on $n$-bit strings, the formula says that exactly one symbol encodes a tape-letter/state pair, and all the other symbols encode just tape-letters. It also says that $\bin(0) = 0^n$, $\bin(2^n-1) = 1^n$, and for every $j \leq n$, the $j$th bit in a block is flipped in the next block iff all bits strictly lower in this block are $1$s. Thus, the formula $\Phi_{conf}$ can be defined as
$\texttt{init} \land \malways (\#' \limp \mnext \conf)$ where $\texttt{init}$ is a formula that encodes the initial configuration (which can be hard-coded by a polynomial sized formula by explicitly specifying the first $|x|$ blocks, and that the rest of the blocks in the configuration contain the encoding of the blank tape symbol). Writing linearly-sized \LTLf formulas
$\Phi^{odd}_{tran}, \Phi^{even}_{tran}$, $\Phi_{num}$, and $\Phi_{acc}$ poses no particular problem.

It remains to show how to build the formula $\Phi_{chal}$. It will be the conjunction of two formulas $\Phi^1_{chal}$ and $\Phi^2_{chal}$. The first handles the first challenge, and the second handles all the rest. We now show how to build the second (the first is similar). Define $\Phi^2_{chal}$ as:
\[
 \malways \bigwedge \left[(\texttt{cha} \land \texttt{cur}_{x} \land \texttt{nx}_y \land \texttt{nxnx}_z \land \texttt{tr}_t) \limp \texttt{img}_{y'}\right]
\]
where the conjunction is over tuples $(x,y,z,t,y')$ such that 
applying the transition $t$ to the triple of tape-contents $xyz$ (including a possible state) results in the tape content $y'$ of the middle cell (e.g., for $t = (q,l,q',l',R)$, if $x = (l,q)$ then $y' = (y,q')$, if $x = l$ then $y' = y$, etc.).
Intuitively, $\texttt{cha}$ expresses that we are currently at the start of a block of a configuration, say $C_i$, whose number is one less than the challenge number encoded by $K_{i+1}$; the formula $\texttt{cur}_{x}$ expresses that the symbol in the current block is $x$; the formula $\texttt{nx}_y$ expresses that the symbol in the next block is $y$; the formula $\texttt{nxnx}_z$ expresses that the symbol in the block after that is $z$; the formula $\texttt{tr}_t$ says that $T_{i+1}$ encodes the transition $t$; and the formula $\texttt{img}_{y'}$ says that the block whose number is encoded by $K_{i+1}$ in the configuration $C_{i+1}$ is $y'$.

\newcommand{\mscan}{{\mathop{\J}}}

We use the following shorthand, that can scan the string for patterns: define $\phi_1 \mscan^1 \phi_2 := (\lnot \phi_1) \muntil (\phi_1 \land \mnext \phi_2)$ and
$\phi_1 \mscan^{2} \phi_2 := \phi_1 \mscan^1 (\phi_1 \mscan^1 \phi_2)$. Intuitively, $\phi_1 \mscan^i \phi_2$ means $\phi_2$ holds one step after the $i$th occurrence of $\phi_1$.

Formally, define $\texttt{cha}$ as:
\[
\% \land \left[(\%) \mscan^2 \left(\bigwedge_{i \in [0, n)} \bigwedge_{b \in \{0,1\}} \left[\mnext^i b \iff \#'' {\mscan}^2 \mnext^i b\right] \right)\right].
\]
Define $\texttt{cur}_{x}$ as:
$
(\$) \mscan^1 (\wedge_{i: 0 \leq i < m} \mnext^i b_i)
$
where $b_1 b_2 \cdots b_m$ encodes the symbol $x$, and define $\texttt{nx}_y$ and $\texttt{nxnx}_z$ similarly.
Define $\texttt{tr}_t$ as:
$
(\#) \mscan{^1} (\wedge_{i: 0 \leq i < m} \mnext^i t_i)
$
where $t_1 t_2 \cdots t_m$ encodes the symbol $t$.
Define
$\texttt{img}_{y'}$ as:
\[
\left[\#'\right] \mscan{^1}  \left[ (\texttt{match} \limp \mnext^n \wedge_{i: 0 \leq i < m} \mnext^i y'_i) \muntil {\#''}  \right]
\]
where $y'_1 y'_2 \cdots y'_m$ encodes the symbol $y'$, and
$\texttt{match}$ is
\[
\bigwedge_{i \in [0,n)} \bigwedge_{b \in \{0,1\}} (\mnext^i b \iff (\#'') \mscan{^1} (\mnext^i b)).
\]
Intuitively, it says that in the next configuration, if a block number equals the challenge number, then the block symbol should be $y'$. This completes the construction of the goal $\Phi$.
This completes the proof for the case of no fairness. 

For stochastic and state-action fairness, observe that a) if $M$ accepts $x$ then, already with no fairness assumptions, there is a solution, and b) if $M$ rejects $x$, then for every policy $f$, the environment can, within a finite number of steps, prevent any hope of satisfying the goal: either by exposing that the agent is cheating in the simulation, or by reaching a rejecting configuration. Since every finite $f$-trace can be extended to a fair infinite $f$-trace, the policy $f$ is not a solution to the state-action fair planning problem, nor is it a solution to the stochastic fair planning problem since the set of infinite $f$-traces that extend this finite $f$-trace has positive probability.

\section{Related Work and Discussion} \label{sec:discussion}
We have discussed how the distinction between stochastic- and state-action fairness is so-far missing from the planning/AI literature. On the other hand, as we now discuss, this distinction is present in the verification literature. 

\subsubsection{Related work in verification} 
Early work in verification was motivated by the problem of providing formal methods (such as proof-systems or model-checking algorithms) to reason about probabilistic concurrent systems. As such, some effort was made to abstract probabilities and capture stochastic fairness by language-theoretic properties. In fact, sophisticated forms of language-theoretic fairness were introduced to do this~\cite{Pnueli:IC93,Baier:IPL98}, since simple language-theoretic notions (similar to state-action fairness) were known not to capture stochastic fairness~\cite{Pnueli:STOC83}. 

A comprehensive study of fairness in reactive systems is provided in~\cite{Volzer:JACM12} where fairness is characterized language-theoretically, game-theoretically, topologically, and probabilistically. Fairness is used in verification of concurrent systems in order to prove liveness properties, i.e., that something good will eventually happen. The limitations of fairness for proving liveness properties, as well as ways to overcome these limitations, are analysed in~\cite{Glabbeek:ACMCS19}.

The verification literature on probabilistic concurrent programs typically considers policies as schedulers. In particular, the central decision problem there is different to the planning problem: it asks whether \emph{every} (rather than \emph{some}) policy $f$ almost-surely enforces the temporally-extended goal~\cite{Vardi:FOCS85,Pnueli:IC93,Bianco:FSTTCS95,Courcoubetis:JACM95}. Just as we used Rabin conditions to capture state-action unfair traces, one can use the dual Street condition to capture stochastically-fair traces~\cite{Vardi:FOCS85}. 

Generalizations of strong-cyclic solutions, other than those that use state-action fairness, have been studied using automata-theory. For instance,~\cite{Pistore:JAIR07} consider that $f$ is a solution if every finite $f$-trace can be extended to an infinite $f$-trace satisfying the given \LTL formula. However, such a notion is different from a solution assuming state-action fairness (the example in Proposition~\ref{prop:cex} shows this). Also,~\cite{Vardi:CAV95} studies a variation of \LTL synthesis assuming a fair scheduler, where the transitions over a given set of states are assumed to be implicitly encoded in \LTL.

\subsubsection{Discussion}

While stochastic fairness admits well-behaved algorithms, it is not clear that language-theoretic fairness does. For the moment, even for the case of \LTLf goals, our algorithm (Section~\ref{sec:algorithms-and-upper-bounds}) requires automata over infinite traces to deal with state-action fairness, which itself is a property of infinite traces. Unfortunately, algorithms for automata over infinite traces are not as easy to implement as for finite traces~\cite{DFogartyKVW13}. 
 Nonetheless, we hope that our new algorithm, which suggests the importance of planning for Rabin goals, spurs the planning community to devise translations and heuristics for solving these.

\clearpage
\bibliographystyle{aaai}
\bibliography{references}

\end{document}